\documentclass[letterpaper, 10 pt, journal, twoside]{IEEEtran}

\usepackage{enumitem}
\usepackage{amsmath,amsfonts}
\usepackage{amsthm}
\newtheorem{assumption}{Assumption}
\newtheorem{theorem}{Theorem}
\newtheorem{proposition}{Proposition}
\usepackage{multirow}
\usepackage{graphics}
\usepackage{graphicx}
\usepackage{algorithm}
\usepackage{algcompatible}
\usepackage{subfigure}
\usepackage{bbm}
\usepackage{hyperref}
\hypersetup{
    colorlinks=true,
    linkcolor=cyan,
    filecolor=magenta,      
    urlcolor=cyan,
}
\usepackage{amssymb}

\def\BibTeX{{\rm B\kern-.05em{\sc i\kern-.025em b}\kern-.08em
    T\kern-.1667em\lower.7ex\hbox{E}\kern-.125emX}}

\begin{document}
\title{Track-centric Iterative Learning for Global Trajectory Optimization in Autonomous Racing}

\author{Youngim Nam, Jungbin Kim, Kyungtae Kang, and Cheolhyeon Kwon}

\maketitle
\begin{abstract}
This paper presents a global trajectory optimization framework for minimizing lap time in autonomous racing under uncertain vehicle dynamics. Optimizing the trajectory over the full racing horizon is computationally expensive, and tracking such a trajectory in the real world hardly assures global optimality due to uncertain dynamics. Yet, existing work mostly focuses on dynamics learning at the tracking level, without updating the trajectory itself to account for the learned dynamics. To address these challenges, we propose a track-centric approach that directly learns and optimizes the full-horizon trajectory. We first represent trajectories through a track-agnostic parametric space in light of the wavelet transform. This space is then efficiently explored using Bayesian optimization, where the lap time of each candidate is evaluated by running simulations with the learned dynamics. This optimization is embedded in an iterative learning framework, where the optimized trajectory is deployed to collect real-world data for updating the dynamics, progressively refining the trajectory over the iterations. The effectiveness of the proposed framework is validated through simulations and real-world experiments, demonstrating lap time improvement of up to $20.7\%$ over a nominal baseline and consistently outperforming state-of-the-art methods.

\end{abstract}


\section{Introduction}
\IEEEPARstart{I}{n} recent years, autonomous racing has gained significant attention as a promising means to showcase cutting-edge autonomous driving technologies, such as Roborace, Indy Autonomous Challenge, and F1tenth \cite{survey_racing}. 
The primary objective of autonomous racing is to minimize lap time while ensuring safety at the limits of handling. This can be formulated as a full-horizon trajectory optimization problem subject to nonlinear vehicle dynamics and track geometry constraints. 
However, this well-posed objective entails twofold challenges: (i) the long-horizon optimization coupled with nonlinear constraints makes the problem computationally intractable \cite{computation_issue}; and (ii) the considered nominal vehicle dynamics may not fully capture the true dynamics in the real world, so any model mismatch can result in a gap from the ideal lap time \cite{survey_model}. 
Most existing studies have tackled one of these challenges, often compromising each other \cite{Planning_mincurv, Control_onlinegp}.

To ease the computational complexity of the full-horizon trajectory optimization, many studies have simplified its formulation, either by decomposing the optimization into multiple sub-optimizations, such as optimizing path and then velocity \cite{Planning_twostage}, or by reducing the dimension of the problem via trajectory parameterization \cite{Planning_bayesian}.
However, they have often resorted to simplified dynamics for computational efficiency, which fail to address the true dynamics in the real world (See Fig.~\ref{fig:conceptart}(a)).
To better account for realistic dynamics, several studies have recast the problem to a receding horizon fashion in the spirit of Model Predictive Control (MPC) \cite{Control_mintime, Control_mintime2}. 
Nonetheless, the performance of MPC ultimately depends on the accuracy of its dynamics.
Therefore, the mismatch between the MPC dynamics and the true dynamics leads the actual trajectory to deviate from the predicted one, making the optimal lap time unattainable (See Fig.~\ref{fig:conceptart}(b)).

\begin{figure}[!t]
    \centering
\includegraphics[width=0.45\textwidth]{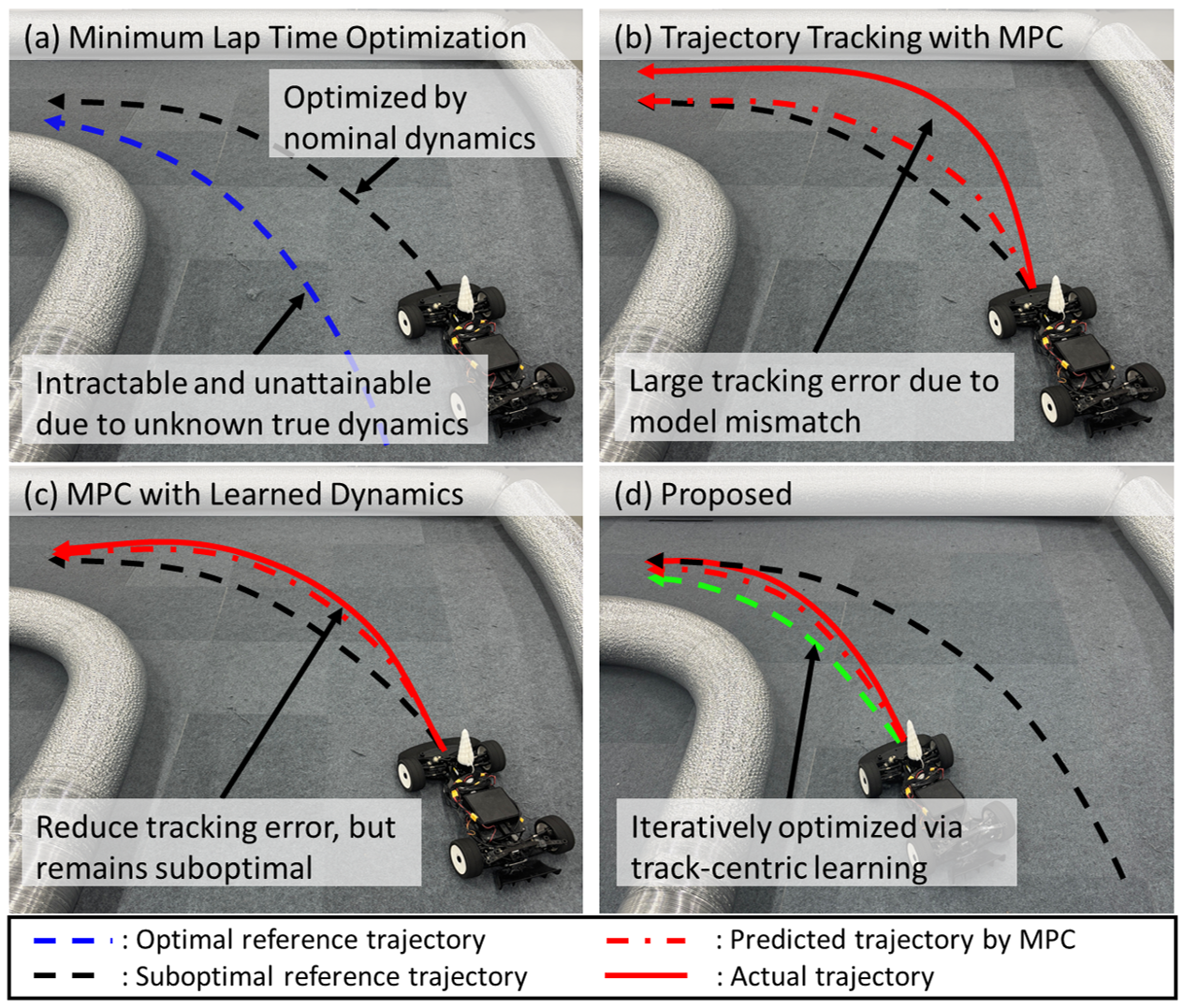}
    \caption{(a) Trajectory optimization with an inaccurate nominal model yields a suboptimal trajectory; (b) Standard MPC exhibits large tracking errors due to model mismatch; (c) MPC with learned dynamics reduces tracking error but follows a suboptimal reference; and (d) Our proposed framework iteratively refines the trajectory to improve the lap time.}
    \label{fig:conceptart}
\end{figure}

To deal with uncertain dynamics, prior research has primarily focused on compensating for the model mismatch at the control level (i.e., reference trajectory tracking). 
Data-driven methods such as Gaussian Process (GP) regression \cite{Control_gp} and Neural Networks (NNs) \cite{Control_NN} learn the residual dynamics between the nominal and true vehicle dynamics to correct prediction errors.
However, as the reference trajectory to be tracked is still generated from the nominal dynamics, the resulting vehicle behavior remains suboptimal in the global planning sense (See Fig.~\ref{fig:conceptart}(c)).
Recently, \cite{doublegp} has proposed an iterative framework that updates dynamics and applies it to both trajectory optimization and tracking.
Yet, this proposed method is incapable of directly accounting for the learned dynamics into the trajectory optimization.
The learned dynamics is evaluated prior to optimization, which is ineffectual for minimizing the real lap time.

Motivated by the aforementioned limitations, we propose a new learning paradigm for minimum lap time trajectory optimization under uncertain dynamics, namely a track-centric learning: \emph{unlike existing work that focuses on learning vehicle dynamics, i.e., dynamics-centric learning, we directly learn and optimize the trajectory for a given race track.}
This paradigm is particularly advantageous when collecting real-world data is costly or unsafe. 
Instead of learning the dynamics across the entire driving envelope, our approach focuses on refining the dynamics along the trajectory being optimized, specific to the given track.
Notably, it requires far fewer real-world trials while still capturing the behavior most relevant to lap time performance.

To this end, our framework consists of two key components.
We first introduce a wavelet-based parameterization to represent trajectories in a low-dimensional space.
Wavelets were originally developed to effectively modulate time signals with both global trends and local variations, a property that has recently led to their adoption in robotics for representing complex trajectories \cite{wavelet, wavelet2}.
Unlike conventional curve parameterizations (e.g., splines or polynomials) that impose rigid geometric structures, wavelets offer a flexible parameterization well-suited for capturing the mixed smooth and abrupt variations in racing tracks.
Second, we employ Bayesian Optimization (BO) to explore the wavelet parameter space through a track-centric surrogate model that maps trajectory parameters to expected lap time.
This surrogate model is evaluated through simulations, where each candidate trajectory is run on the learned vehicle dynamics.

Based on the wavelet-based trajectory parameterization and the track-centric surrogate model, we establish an iterative learning framework that progressively refines the trajectory: (i) the vehicle is driven on the real race track to collect new data; (ii) the vehicle dynamics is updated to improve simulation fidelity; and (iii) trajectory optimization is performed using BO within the simulation-enabled surrogate model.
As iterations proceed, the learned dynamics become more accurate along the trajectory being optimized, enabling more reliable lap time evaluation for trajectory refinement.
The major contributions are summarized as follows: 
\begin{itemize}[noitemsep]
   \item An iterative learning framework is developed to conduct data collection, dynamics update, and trajectory optimization, all centered on the lap time performance of a given race track;
   \item A wavelet-based track-agnostic parameterization of the trajectory is introduced to enable low-dimensional and computationally efficient trajectory optimization;
   \item A track-centric surrogate model is established to encompass learned dynamics, whereby evaluating the lap time directly from the trajectory; 
   \item Theoretical analysis is carried out to characterize error reduction over the learning iterations, offering a formal guarantee of performance improvement under the proposed framework; and
   \item The effectiveness of the proposed framework is validated through extensive simulations and real-world experiments, demonstrating faster lap time compared to existing methods.
\end{itemize}

\section{Related Work}
\subsection{Trajectory Optimization}
Early works on full-horizon trajectory optimization have primarily focused on reducing its computational burden through simplified formulations.
One common strategy is to decompose the original lap time minimization into multiple sub-optimizations \cite{Planning_twostage}, where a geometric path is first optimized based on the proxy objectives such as minimum curvature \cite {Planning_mincurv, Planning_mincurv2}, and the corresponding velocity profile is then optimized along the path \cite{Planning_velprofile}.
Another approach is trajectory parameterization, in which the trajectory is represented using a finite set of parameters (e.g., lateral offsets along a set of predefined waypoints on the track) \cite{Planning_bayesian}.
This representation reduces the dimension of the search space, enabling efficient exploration for black-box optimizers such as BO or genetic algorithms \cite{Planning_GA}.
Yet, to maintain computational efficiency, these methods often rely on kinematic or quasi-dynamic models, since incorporating high-fidelity dynamics into the full-horizon trajectory optimization is computationally unaffordable.
To accommodate the high fidelity vehicle dynamics in the trajectory optimization, several works have adopted MPC to reformulate the problem into a receding horizon fashion \cite{Control_MPC, Control_mintime3}.
However, the limited horizon of MPC inherently restricts the solution to be locally optimal.
More importantly, the potential model mismatch issue still persists in the vehicle dynamics. 
Without fully capturing the true vehicle dynamics, MPC falls short of the optimal trajectory even in local aspect.

\subsection{Data-driven Approaches for Model Mismatch}
Most existing studies addressing model mismatch are grounded in data-driven regime, usually applied at the controller level.
A common approach is to use GP \cite{Control_onlinegp} or NNs \cite{Control_NN} to learn residual dynamics, which can be incorporated into the controller (e.g., MPC) \cite{Control_gp}.
Instead of explicitly learning residual dynamics, some approaches have implicitly addressed model mismatch by adapting controller parameters, such as the cost function or constraints, using collected data.
For instance, learning-based MPC frameworks construct data-driven safe sets and terminal costs \cite{Control_NMPC, Control_NMPC2}, while BO is employed to tune cost weights or constraint margins \cite{Control_BO, Control_BO2}.
Despite these advancements, control actions are still subject to receding-horizon MPC, and thus the resulting trajectory remains suboptimal.

Another group of studies has explored the model-free approach, such as Reinforcement Learning (RL).
RL bypasses the need for an explicit dynamics model by directly learning control actions through interaction with the environment.
However, without a world model, the policy must implicitly reason about the entire lap directly from local observations, which is inherently sample-inefficient.
To address this challenge, recent works employ trajectory-conditioned RL, where the learning target is recast to trajectory tracking, demonstrating significantly improved performance over plain RL \cite{RL}.
Taken together, these developments highlight the importance of explicitly representing trajectories for better performance.
Yet, existing approaches treat the trajectory as a given reference and do not refine it while learning.
To the best of our knowledge, \cite{doublegp} is the only work that considers learned vehicle dynamics in full-horizon trajectory optimization.
However, directly integrating GP-based residual dynamics into a full-horizon optimization process is computationally prohibitive, if not intractable, and is therefore evaluated only prior to optimization.
In sum, due to all the aforementioned limitations in the existing works, \emph{computationally tractable} full-horizon trajectory optimization under \emph{uncertain vehicle dynamics} remains a largely open problem.

\section{Problem Formulation}

\subsection{Vehicle Dynamics}
We consider a bicycle model expressed in the Frenet frame with respect to the track centerline \cite{Control_mintime}, as illustrated in Fig.~\ref{fig:dyn}. 
The vehicle state is defined as $\mathbf{x} := [s, e_y, e_{\psi}, v_{x}, v_{y}, w ]^{\!\top} \in \mathbb{R}^{n_x}$
where $s$ is the arc-length along the track centerline, $e_y$ is the lateral offset, $e_\psi$ is the heading error, $v_x$ and $v_y$ are the longitudinal and lateral velocities in the body frame, and $w$ is the yaw rate, respectively. 
The control input is defined as $\mathbf{u} := [a, \delta ]^{\!\top} \in \mathbb{R}^{n_u}$ where $a$ and $\delta$ are the longitudinal acceleration and steering angle. 
The resulting continuous-time nonlinear equations are given by:
\begin{subequations} \label{eq.continuous_dyn}
\begin{align}
&\dot{s} = \frac{v_x \cos(e_\psi) - v_y \sin(e_\psi)}{1-k(s)e_y} \label{eq.continuous_dyn_a}\\
&\dot{e}_{y} = v_x \sin(e_\psi) + v_y \cos(e_{\psi}) \\
&\dot{e}_{\psi} = w - \frac{v_x \cos(e_{\psi}) - v_y \sin(e_{\psi})}{1-k(s)e_y}k(s) \\
&\dot{v}_x = a - \frac{1}{m}(F_{yf}\sin(\delta) - m w v_y) \\
&\dot{v}_y = \frac{1}{m}( F_{yf} \cos(\delta) + F_{yr}-m wv_x) \\
&\dot{w} = \frac{1}{I_z}(L_f F_{yf} cos(\delta) - F_{yr}L_r)
\end{align}
\end{subequations}
where $m$ and $I_z$ respectively represent the vehicle mass and yaw moment of inertia. 
$L_f$ and $L_r$ are the distances from the center of gravity to the front and rear wheels, and $k(s)$ describes the curvature of the centerline along $s$. 
The lateral tire forces $F_{yf}$ and $F_{yr}$ are modeled using a simplified Pacejka tire model \cite{pacejka}, parameterized by the stiffness factor $B$, shape factor $C$, and friction coefficient $\mu$. 
In compact form, the continuous-time nominal dynamics can be expressed as:
\begin{equation} \label{eq.nominal_dyn_c}
\dot{\mathbf{x}} = f(\mathbf{x}, \mathbf{u}).
\end{equation}
\begin{figure}[!t]
    \centering
    \includegraphics[width=0.45\textwidth]{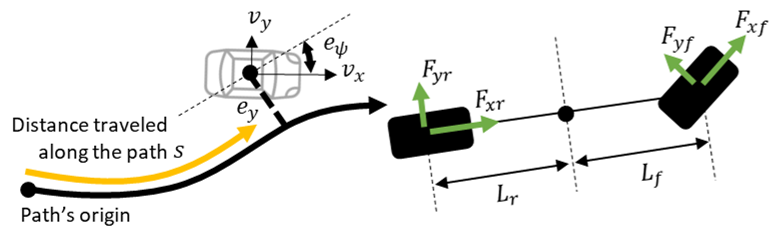}
    \caption{Schematic of the bicycle model described in the Frenet frame.}
    \label{fig:dyn}
\end{figure}
\subsection{Minimum Lap Time Optimization Problem} \label{section:Mintime}
The main goal in autonomous racing is to complete the track as fast as possible, i.e., to minimize the total lap time $T$.
To explicitly handle track boundaries and reduce the problem dimensionality, the optimization problem is formulated over the spatial coordinate $s$ rather than time $t$ \cite{Control_MPC}, i.e., $T = \int \mathrm{d}t = \int \frac{\mathrm{d}t}{\mathrm{d}s} \mathrm{d}s$.
By discretizing the track into $N_s$ arc-length segments, this problem is formulated subject to vehicle dynamics and track constraints as follows:
\begin{subequations} \label{eq.laptime_opt}
\begin{align}
\min_{\{\mathbf{x}_k, \mathbf{u}_k\}_{k=0}^{N_s-1}} \quad & T = \sum_{k=0}^{N_s-1} \frac{1}{\dot{s}_k} \Delta s \label{eq.laptime_obj} \\
\text{s.t.} \quad & \mathbf{x}_{k+1} = f_s(\mathbf{x}_k, \mathbf{u}_k) \label{eq.dyn_constraints} \\
& \mathbf{x}_1 = \mathbf{x}_{N_s} \\
& \mathbf{x}_k \in \mathcal{X} \\
& \mathbf{u}_k \in \mathcal{U}, \quad \forall k = 0, \ldots, N_s-1.
\end{align}
\end{subequations}
where $\mathcal{X}$ and $\mathcal{U}$ are the feasible state and input sets, respectively. $\Delta s$ is the arc-length discretization step, and 
$f_s(\cdot)$ represents the discrete-space dynamics, numerically integrated from \eqref{eq.nominal_dyn_c} over $\Delta s$.
Notably, the solution to \eqref{eq.laptime_opt} is inherently sensitive to model mismatch in vehicle dynamics, resulting in suboptimal trajectories exhibiting longer lap time.

\subsection{Gaussian Process-based Residual Dynamics Learning} \label{section_GP}
To compensate for such model mismatch between the nominal model \eqref{eq.nominal_dyn_c} and the true vehicle dynamics, we model the residual dynamics using GP regression.
\eqref{eq.continuous_dyn} indicates that  the velocity-related states $v_x, v_y$, and $w$ are directly influenced by uncertain parameters such as tire forces.
Accordingly, the residual term is defined to capture the unmodeled velocity-related dynamics, and the regression features are selected as $\mathbf{z} := [v_x, v_y, w, a, \delta]$ \cite{Control_NMPC2}.
The residual dynamics is modeled as a zero-mean GP prior:
\begin{equation} \label{eq.gp_prior}
 \mathbf{y}=g(\mathbf{z}) \sim \mathcal{GP}(0, k(\mathbf{z}, \mathbf{z}'))   
\end{equation}
where $k(\mathbf{z}, \mathbf{z}')$ is a kernel function defining the covariance, designed as a standard Radial Basis Function (RBF) kernel.

Given a set of $N_D$ collected state-input samples $\{\mathbf{x}_i, \mathbf{u}_i\}$ and their corresponding state derivative measures ${\dot{\mathbf{x}}}_{i}^{\text{true}}$, the GP training dataset $\mathcal{D}_g = \{\mathbf{z}_i, \mathbf{y}_i\}_{i=1}^{N_D}$ is acquired. 
Each target residual $\mathbf{y}_i$ is obtained from the discrepancy between the measured and nominal state:
\begin{equation}
    \mathbf{y}_i = S_v\!\left({\dot{\mathbf{x}}}_i^{\text{true}} - f(\mathbf{x}_{i},\mathbf{u}_{i})  \right)
\end{equation}
where $S_v: \mathbb{R}^{n_x} \rightarrow \mathbb{R}^{3}$ is a linear map that extracts $({v}_x, {v}_y, w)$ from $\mathbf{x}$.
Accordingly, the GP prior defined in \eqref{eq.gp_prior} is trained on $\mathcal{D}_g$, resulting in the posterior mean $\mu_g(\mathbf{z})$ and variance $\sigma_g^2(\mathbf{z})$.
By incorporating the learned residual dynamics into \eqref{eq.nominal_dyn_c}, the augmented vehicle dynamics can be represented as:
\begin{equation} \label{eq.gp_model}
    \dot{\mathbf{x}} = f(\mathbf{x}, \mathbf{u}) + B_g g(\mathbf{z})
\end{equation}
where $B_g \in \mathbb{R}^{n_x \times 3}$ is the corresponding injection matrix that maps the predicted residuals $\mathbf{y}$ 
to the velocity-related state derivatives, satisfying the consistency condition $S_v B_g = I_3$.
However, directly substituting \eqref{eq.gp_model} into the minimum lap time optimization \eqref{eq.laptime_opt} yields a stochastic and computationally intractable problem due to the high-dimensional state-input space over full-horizon steps and the non-parametric nature of GP.
To address this, we reformulate the trajectory optimization using a wavelet-based parameterization to reduce the search dimension of problem; and employ BO over a track-centric surrogate model to solve for the optimal global trajectory under uncertain vehicle dynamics.
\vspace{-0.5em}
\section{Algorithm Development}
\begin{figure}[!t]
    \centering
\includegraphics[width=0.45\textwidth]{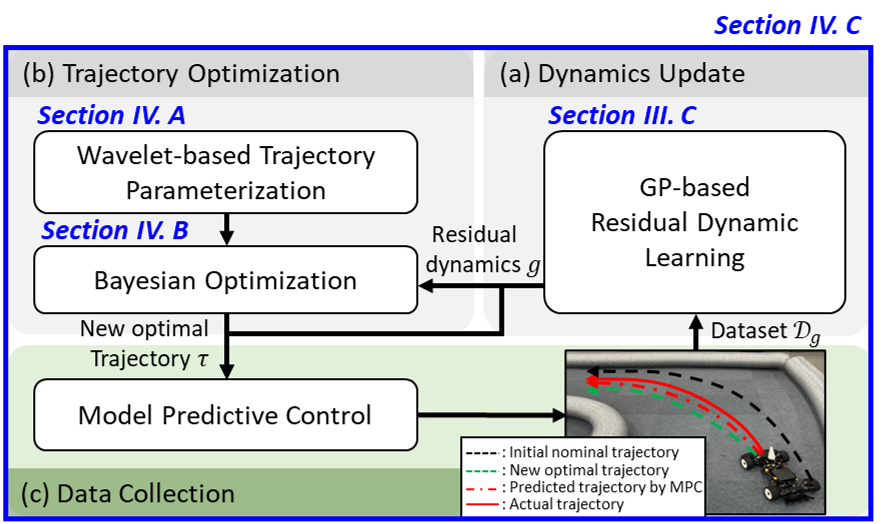}
    \caption{Overall architecture of iterative track-centric learning framework.}
    \label{fig:algorithm}
\end{figure}
\subsection{Wavelet-based Trajectory Parameterization} \label{section:wavelet}
To make the global trajectory optimization computationally affordable, we represent the trajectory $\tau$ by a low-dimensional parameter vector $\theta$ in light of wavelet transformation.
Specifically, we parameterize two key profiles of trajectory\footnote{$e_y$ and $v_x$ serve as the principal design variables of racing trajectory, while all other vehicle states (e.g., $e_\psi, v_y, w$) are physically subject to these two variables in view of the vehicle dynamics \cite{survey_model}.}: the lateral deviation from the track centerline $e_y(s)$ and the longitudinal velocity profile $v_x(s)$ along the arc-length $s$.
While the trend of these profiles varies significantly across different race tracks, conventional curve parameterization techniques often require manual tuning, such as model order or control point density, to capture such trend.
More fundamentally, under a limited number of parameters, these parameterizations have limited expressivity that can fail to represent the optimal trajectory, critically near apexes.
In contrast, wavelets enable a flexible representation of trajectory through a multi-resolution decomposition, as illustrated in Fig. \ref{fig:wavelet}.
Specifically, each profile is decomposed into approximation coefficients $c$, which capture the global structure, and detail coefficients $d$, which encode localized variations and higher-frequency terms \cite{wavelet}: 
\begin{subequations}
\begin{align}
e_y(s) &= \sum_{k} c^{e_y}_{L,k}\, \phi_{L,k}(s) + \sum_{l=0}^{L-1}\sum_{k} d^{e_y}_{l,k}\, \psi_{l,k}(s)\\
v_x(s) &= \sum_{k} c^{v_x}_{L,k}\, \phi_{L,k}(s) + \sum_{l=0}^{L-1}\sum_{k} d^{v_x}_{l,k}\, \psi_{l,k}(s)
\end{align}
\end{subequations}
where $\phi_{L,k}$ and $\psi_{l,k}$ respectively denote the scaling and wavelet basis functions.
$L$ represents the maximum decomposition level, and $k$ indexes the translation of each basis function within that level.

Since the trajectory is defined over discrete arc-length samples $s = \{s_0, s_1, \cdots, s_{N_s-1}\}$, the wavelet coefficients $(c,d)$ are obtained by applying Discrete Wavelet Transform (DWT) to the discretized $e_y$ and $v_x$ profiles.
At the initialization step, we apply DWT to the nominal trajectory $\tau^0$ obtained by solving \eqref{eq.laptime_opt}, yielding the initial coefficient sets $(c^0, d^0)$.
Among these coefficients, the only coarsest-level approximation coefficients are selected as optimization variables, while the other associated detail coefficients are kept fixed at their initial values throughout the optimizations. 
Such selection is motivated by the fact that lap time performance is largely influenced by the global trend of the racing trajectory\footnote{High-frequency trajectory variations tend to induce large curvature peaks, which increase lap time and violate tire forces and jerk limits \cite{Planning_mincurv}.}.
Correspondingly, the selected coefficient vectors $c^{e_y}_{L,k}$ and $c^{v_x}_{L,k}$ are concatenated to form:
\begin{equation}\label{eq:theta_vec}
\theta =
\begin{bmatrix}
c^{e_y}_{L,1}, \dots, c^{e_y}_{L,N_{e_y}},\;
c^{v_x}_{L,1}, \dots, c^{v_x}_{L,N_{v_x}}
\end{bmatrix}^{\!\top}
\in \mathbb{R}^{N_\theta}
\end{equation}
where $N_{e_y}$ and $N_{v_x}$ denote the numbers of retained coefficients at the coarsest level for $e_y$ and $v_x$, respectively.
This parameterization provides a low-dimensional search space while preserving the qualification of the trajectory.

\begin{figure}[!t]
    \centering
    \includegraphics[width=0.48\textwidth]{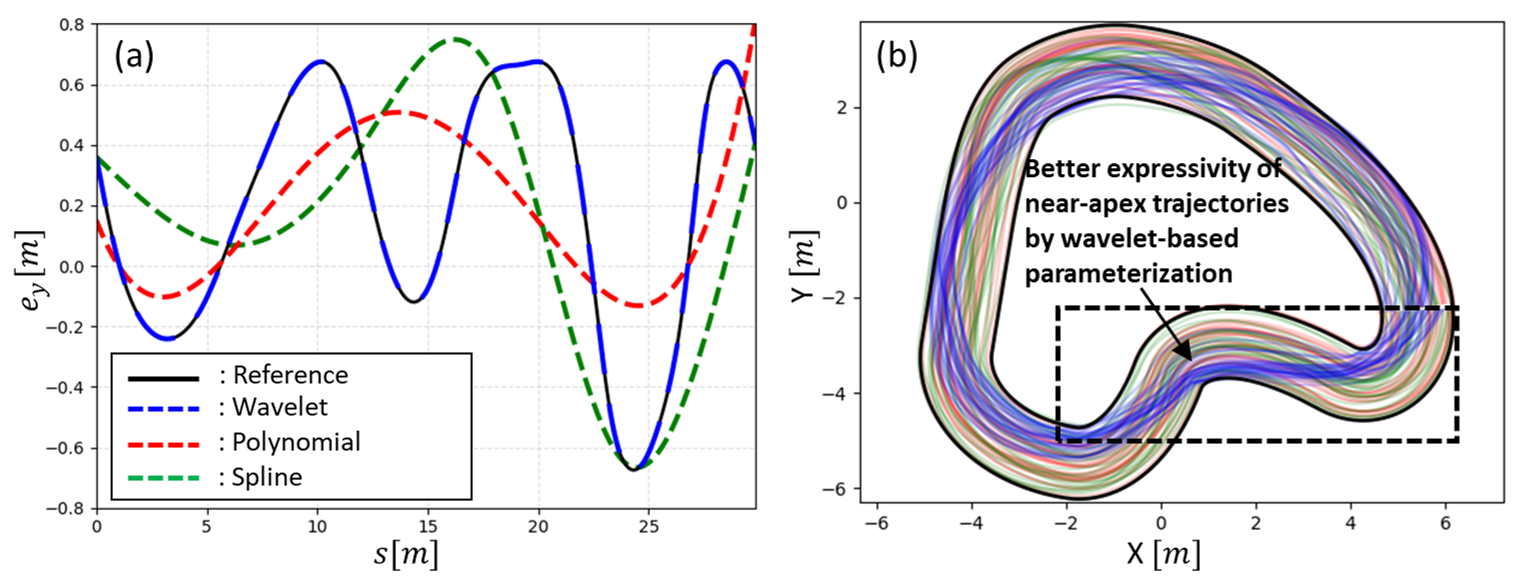}
    \caption{Expressivity gap in conventional vs. wavelet-based trajectory parameterizations: under a limited parameter number, (a) wavelet-based parameterization can accurately capture different trajectory trends; and (b) can represent rich samples near the optimal trajectory.}
    \label{fig:wavelet}
\end{figure}

\subsection{Bayesian Optimization for Lap Time Minimization} \label{section_BO}
Within the wavelet-based parameterization, the lap time minimization problem can be reformulated as:
\begin{equation}
    \theta^{*} = \arg\min_{\theta \in \Theta} J(\theta)
\end{equation}
where $J(\theta)$ denotes the total lap time associated with $\theta$, and $\Theta$ denotes the feasible set constrained by the track geometry and the limits in vehicle dynamics.
As this evaluation is computationally expensive and does not admit analytic gradients, we treat $J(\theta)$ as a black-box function and employ BO to search for $\theta^*$ in a sample-efficient manner\footnote{Other gradient-free optimization methods (e.g., genetic algorithm, CMA-ES, etc.) can be also used in the proposed framework. However, they are in general hard to show theoretical guarantees, if any \cite{EA}.}.
Each evaluation of $J(\theta)$ involves a closed-loop simulation whereby: (i) the candidate trajectory is reconstructed via inverse DWT from $\theta$; (ii) the vehicle with the learned residual dynamics is simulated to track this candidate trajectory; and (iii) the resulting lap time is measured, all in simulation level.
This process yields a scalar lap time corresponding to $\theta$, which plays as a noisy observation of the underlying objective function.

To model the relationship between $\theta$ and $J(\theta)$ on a given track, BO employs a probabilistic surrogate model $\hat{J}(\theta)$ using GP regression:
\begin{equation}
    \hat{J}(\theta) \sim \mathcal{GP}\big(\mu(\theta), \sigma^2(\theta)\big)
\end{equation}
where $\mu(\theta)$ and $\sigma^2(\theta)$ respectively denote the predicted lap time and its associated model uncertainty.
Based on this track-centric surrogate model, BO proposes and evaluates the candidate parameters over recursive cycles. 
At each recursion $n$, given the current dataset 
$\mathcal{D}_{BO} = \{(\theta_i, J(\theta_i))\}_{i=1}^{n}$, 
the GP posterior is updated to produce the predictive mean $\mu_n(\theta)$ and variance $\sigma_n^2(\theta)$. 
The next candidate trajectory parameter vector $\theta_{n+1}$ is then selected according to the Lower Confidence Bound (LCB) criterion:
\begin{equation} \label{eq:bo_acq}
\theta_{n+1} = \arg\min_{\theta \in \Theta} \big[\mu_n(\theta) - \beta^{1/2}\sigma_n(\theta)\big]
\end{equation}
where $\beta>0$ balances the exploration–exploitation trade-off. 
Once $\theta_{n+1}$ is obtained, its corresponding lap time $J(\theta_{n+1})$ is evaluated through the closed-loop simulation.
The new observation $(\theta_{n+1}, J(\theta_{n+1}))$ is then incorporated into the dataset $\mathcal{D}_{BO}$. 
This process is repeated for $N_{\text{BO}}$ recursions, after which the parameter achieving the lowest observed lap time among all evaluated candidates is selected.

\subsection{Iterative Track-Centric Learning Framework}
The BO for trajectory optimization is incorporated into an iterative learning framework that progressively refines the trajectory under uncertain vehicle dynamics, as illustrated in Fig.~\ref{fig:algorithm}.
At iteration $j$, the framework seeks an optimal trajectory $\tau^j$, parameterized by $\theta^{j}$, with respect to the learned dynamics $g^j$.
Each execution step proceeds as follows:
\begin{enumerate}
    \item \textbf{Initialization} (Iteration 0): The process is initialized by solving the minimum lap time problem with nominal dynamics (Sec. \ref{section:Mintime}). The resulting trajectory $\tau^0$ is deployed on the real track to collect the dataset $\mathcal{D}_g^0$.
    
    \item \textbf{Iterative Learning Cycle} (Iteration $j \ge 1$): 
    Starting from the wavelet-based representation $\theta^{0}$ of the initial trajectory $\tau^0$, both dynamics $g^j$ and trajectory $\tau^j$ are updated upon new data collection.
    This consists of the following three-phase cycle:
    \begin{itemize}
        \item[(a)] \textbf{Dynamics Update:} The residual dynamics $g^j$ is trained using all data collected up to the previous iteration, denoted by $\mathcal{D}_g^{j-1}$ (Sec.~\ref{section_GP}).
        
        \item[(b)] \textbf{Trajectory Optimization:} 
       Through the simulation that incorporates the updated residual dynamics $g^j$ and a tracking controller\footnote{In our implementation, the MPC with the learned residual dynamics \cite{Control_gp} is used for trajectory tracking, both within the simulation for surrogate model and during on-track data collection.
        Note that our framework is compatible with any choice of tracking controller, as the trajectory optimization is conducted through controller-in-the-loop simulation.}, BO computes a new optimal trajectory parameter vector $\theta^{j}$ (Sec.~\ref{section_BO}).
        
        \item[(c)] \textbf{Data Collection:} 
        The optimized trajectory $\tau^{j}$ corresponding to $\theta^{j}$ is deployed on the real track and collect new data $\mathcal{D}_g^{\text{new}}$.
        These data are aggregated as 
        $\mathcal{D}_g^{j} \leftarrow \mathcal{D}_g^{j-1} \cup \mathcal{D}_g^{\text{new}}$ for the next iteration.
    \end{itemize}
\end{enumerate}

\subsection{Theoretical Analysis}
This section provides a theoretical foundation on the performance guarantee of the proposed framework. 
Specifically, we analytically show an asymptotic bound assuring that the true lap time of the proposed BO trajectory is near-optimal.
Let $g^*$ denote the true residual dynamics and $J_{g^*}(\theta)$ the true lap time evaluated under $g^*$. 
The corresponding optimal lap time within the parameterized trajectory class is $J^*= \min_{\theta \in \Theta} J_{g^*}(\theta)$.
At iteration $j$, $J_{g^j}(\theta)$ denotes the lap time evaluated under the learned residual dynamics $g^j$.
Without loss of generality, the following assumptions are considered.
\begin{assumption}
\label{assum:model_convergence}
As more data are collected, the learned residual dynamics $g^{j}$ converges to the true residual $g^*$, satisfying
$\|g^{j} - g^*\|_{\infty} \leq \epsilon^{j}$,
where $\epsilon^{j} \to 0$ as $j \to \infty$ \cite{hewing2020learning}.
\end{assumption}
\begin{assumption}
\label{assum:bo_bound}
At each iteration $j$, the BO returns a trajectory parameter vector $\theta^j$ such that:
\[
J_{g^j}(\theta^j) \le \min_{\theta \in \Theta} J_{g^j}(\theta) + \delta(N_{BO})
\]
where $\delta(N_{BO})$ denotes a finite-time suboptimality bound of BO that decreases as the number of evaluations $N_{BO}$ increases \cite{BO_assump}.
\end{assumption}

Assumption~\ref{assum:model_convergence} formalizes the convergence of the learned residual dynamics to the true residual dynamics, while Assumption~\ref{assum:bo_bound} indicates the
suboptimality due to a finite number of BO evaluations at each iteration.
Under Assumption~\ref{assum:model_convergence}, we first bound the lap time evaluation error
incurred when $J(\theta)$ is evaluated under $g^j$ instead of $g^*$.
\begin{proposition}
\label{prop:surrogate_fidelity}
For any $\theta \in \Theta$, the lap time evaluation error satisfies: 
\[
|J_{g^j}(\theta) - J_{g^*}(\theta)| \leq C^{j}\epsilon^{j}
\]
where $C^j > 0$ is a bounded constant.
\end{proposition}
\begin{proof}[Proof]
See Appendix \ref{appendix:proof_prop1}.
\end{proof}

Proposition~\ref{prop:surrogate_fidelity} associates the model mismatch error with the lap time evaluation error for any fixed trajectory parameter.
Together with Assumption~\ref{assum:bo_bound} and the convergence $\epsilon^j \to 0$ by
Assumption~\ref{assum:model_convergence}, we can characterize the bound of true lap time
performance of the trajectory parameters returned by BO.
\begin{theorem} \label{theom}
Suppose Assumptions~\ref{assum:model_convergence} 
and~\ref{assum:bo_bound} hold. 
Then, by Proposition~\ref{prop:surrogate_fidelity}, 
the trajectory parameter $\theta^j$ found at iteration $j$ yields a true lap time
$J_{g^*}(\theta^j)$ that is asymptotically bounded within a neighborhood of the optimal lap time:
\[
\limsup_{j \to \infty} J_{g^*}(\theta^{j}) \leq J^{*} + \delta(N_{BO}).
\]
\end{theorem}

\begin{proof}
From Proposition~\ref{prop:surrogate_fidelity}, we have $J_{g^*}(\theta^{j}) \leq J_{g^j}(\theta^{j}) + C^{j}\epsilon^{j}$. Combining this with Assumption~\ref{assum:bo_bound} yields: 
\begin{equation} \label{eq.proof1}
    J_{g^*}(\theta^{j}) \leq \min_{\theta \in \Theta} J_{g^j}(\theta) + \delta(N_{BO}) + C^{j}\epsilon^{j}.
\end{equation}
Proposition~\ref{prop:surrogate_fidelity} also implies that $J_{g^j}(\theta) \le J_{g^*}(\theta) + C^{j}\epsilon^{j}$ for all $\theta \in \Theta$, and hence:
\begin{equation} \label{eq.proof2}
\min_{\theta \in \Theta} J_{g^j}(\theta) 
\le \min_{\theta \in \Theta} \bigl(J_{g^*}(\theta) + C^{j}\epsilon^{j}\bigr)
= J^* + C^{j}\epsilon^{j}.
\end{equation}
Combining \eqref{eq.proof1} and \eqref{eq.proof2} yields
$J_{g^*}(\theta^{j}) 
\le J^* + 2 C^{j}\epsilon^{j} + \delta(N_{BO}).$
Since $\epsilon^{j} \to 0$ as $j \to \infty$ by 
Assumption~\ref{assum:model_convergence}, we obtain
$
\limsup_{j \to \infty} J_{g^*}(\theta^{j}) \le J^* + \delta(N_{BO}),
$
which completes the proof.
\end{proof}

\section{Experimental Results}
The proposed framework is validated both in simulation and on a real-world 1/10-scale racing platform.
All experiments are conducted from the same initial condition,
$\mathbf{x}_0 = [0\mathrm{m},~0\mathrm{m},~0\mathrm{rad},~ 0\mathrm{m/s},~0\mathrm{m/s},~0\mathrm{rad/s}]^{\!\top},$ and lap time is measured on the second lap to exclude the effects of the initial acceleration phase.
The nominal vehicle dynamics follows the nonlinear bicycle model in \eqref{eq.continuous_dyn}, with parameters set to $B = 1.3, ~ C = 1.5, ~ \mu = 1.2, ~ I_z = 0.024\mathrm{kg\!\cdot\!m^2}, ~ m = 3.0\mathrm{kg}, ~ L_f = 0.14\mathrm{m}, ~ L_r = 0.14\mathrm{m},$ and $g = 9.81\mathrm{m/s^2}$.
The trajectory is discretized into $N_s = 256$ and parameterized using a Daubechies-4 (db4) wavelet with a decomposition level of $L=6$.
This yields $N_{\theta} = 10$ trajectory parameters ($N_{e_y}=5, ~ N_{v_x}=5$), which are optimized via $N_{BO} = 70$ BO recursions.
The proposed algorithm is implemented in Python, and \texttt{FORCESPRO} \cite{forcespro} is used to run the MPC tracking controller at a sampling time of $0.05\,\mathrm{s}$.
For residual dynamics learning, a sparse GP model with 200 inducing points is employed to improve computational efficiency \cite{sparsegp}.

To validate the effectiveness of the proposed algorithm, the comparative analysis and ablation study are conducted with the following baselines:
(i) \textbf{Nominal}: both trajectory optimization and tracking are performed using the nominal dynamics without any learned corrections;
(ii) \textbf{GP-Track}~\cite{Control_gp} (\emph{comparison} 1): GP-based residual dynamics is incorporated into the MPC for tracking while the reference trajectory is set to the one optimized under nominal dynamics;
(iii) \textbf{GP-Opt+Track}~\cite{doublegp} (\emph{comparison} 2): GP-based residual dynamics is accounted for in both trajectory optimization and MPC;
(iv) \textbf{Spline-based} (\emph{ablation} 1): the proposed wavelet-based trajectory parameterization is replaced with a conventional cubic spline parameterization, using 10 control points uniformly spaced along the arc-length; and
(v) \textbf{Non-iterative} (\emph{ablation} 2): sufficient data is collected in advance from arbitrary trajectories on the real track, from which residual dynamics is learned and BO is performed without iterative refinement.
All algorithms are evaluated over 10 iterations, while the \emph{Non-iterative} ablation uses a dataset of equivalent size to ensure a fair comparison.

\subsection{Numerical Simulation}
\begin{figure*}[!t]
    \centering
    \includegraphics[width=0.95\linewidth]{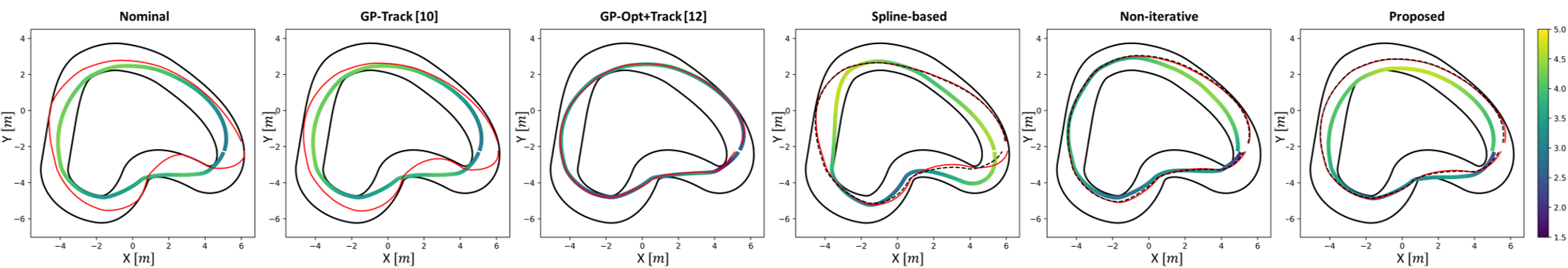}
    \caption{Comparison after 10 iterations for each method. The solid colored line indicates the planned trajectories (colored by speed), the black dotted line shows the trajectories evaluated by the simulation-enabled BO using the learned dynamics, and the red solid line displays the actual trajectories on the real track.}
    \label{fig:result_iteration}
\end{figure*}
To assess performance under diverse conditions, 15 distinct simulation scenarios are considered.
In each scenario, the track geometry is randomly generated, and the true vehicle dynamics is defined by uniformly perturbing the tire characteristics and yaw inertia of the nominal vehicle dynamics within the following ranges: $B^*\!\in\![1.1,1.3]$, $C^*\!\in\![1.3,1.5]$,
$\mu^*\!\in\![0.8,1.2]$, and $I_{z}^*\!\in\![0.014,0.024]~\mathrm{kg\!\cdot\!m^2}$.
For one example scenario, Fig.~\ref{fig:result_iteration} shows the planned (optimized) and actual trajectories at the 10th iteration, and Table~\ref{tab:result_iteration} summarizes the corresponding lap times over iterations.
Here, \(T_{\mathrm{plan}}\) denotes the lap time obtained from trajectory optimization of each methods, \(T_{\mathrm{BO}}\) represents the lap time evaluated by BO, and \(T_{\mathrm{real}}\) is the lap time recorded when tracking the planned trajectory under the true dynamics.

As shown in Fig.~\ref{fig:result_iteration}, the \emph{Nominal} baseline exhibits large tracking errors that eventually lead to collisions.
\emph{GP-Track}~\cite{Control_gp} reduces these errors, but reliable tracking remains challenging.
\emph{GP-Opt+Track}~\cite{doublegp} produces exhibits enhanced tracking by adjusting the planned trajectory through learned dynamics, but lap time improvements remain limited. 
Moreover, a noticeable gap persists between $T_{\mathrm{plan}}$ and $T_{\mathrm{real}}$, since the trajectory optimization does not explicitly account for tracking performance.
In contrast, the proposed method achieves the fastest $T_{\mathrm{real}}$ among all baselines.
This is mainly attributed to the simulation-enabled BO explicitly including the learned dynamics, allowing the trajectory optimization to directly target the reduction of the real lap time.
As the learned dynamics becomes more accurate over iterations, $T_{\mathrm{BO}}$ increasingly aligns with $T_{\mathrm{real}}$, enabling consistent performance improvements.

The assets of the proposed framework are further examined through the ablation study.
The \emph{Spline-based} ablation exhibits degraded performance near high-curvature apex regions, highlighting the advantage of the wavelet representation in capturing sharp curvature despite a limited number of parameters.
The \emph{Non-iterative} ablation shows limited improvement, as learning dynamics across the entire driving envelope leads to insufficient accuracy of $\hat{J}$ in the BO-explored regions.
This underscores the necessity of iterative, track-centric learning, which progressively improves dynamics accuracy along the specific trajectory being optimized.

To evaluate performance across all 15 scenarios, we further quantify the lap time improvement by:
\begin{equation*}
\mathrm{Improvement}~[\%] =
\frac{T_{\mathrm{nom}} - T_{\mathrm{method}}}{T_{\mathrm{nom}}} \times 100
\end{equation*}
where $T_{\mathrm{nom}}$ and $T_{\mathrm{method}}$ respectively denote the lap times obtained by the \emph{Nominal} baseline and each method.
As shown in Fig.~\ref{fig:result_average}, the proposed framework achieves the most lap time improvement across scenarios, with an average gain of $20.7\%$ over the \emph{Nominal} baseline, consistently outperforming all the other baselines.

\begin{table}[!t]
\centering
\caption{Planned, Expected, and Real Lap Time over Iterations}
\label{tab:result_iteration}
\resizebox{0.95\columnwidth}{!}{%
\begin{tabular}{ll|ccccc}
\multicolumn{2}{l|}{}                                                                                      & \multicolumn{1}{c|}{$j=1$}    & \multicolumn{1}{c|}{$j=3$}    & \multicolumn{1}{c|}{$j=5$}    & \multicolumn{1}{c|}{$j=7$}   & $j=10$           \\ \hline
\multicolumn{1}{l|}{\multirow{2}{*}{Nominal}}                                                 & $T_{plan}$ & \multicolumn{5}{c}{6.93s}                                                                                                             \\ \cline{2-7} 
\multicolumn{1}{l|}{}                                                                         & $T_{real}$ & \multicolumn{5}{c}{\textbf{11.2s}}                                                                                                    \\ \hline
\multicolumn{1}{l|}{\multirow{2}{*}{GP-Track \cite{Control_gp}}}                                                & $T_{plan}$ & \multicolumn{5}{c}{6.93s}                                                                                                             \\ \cline{2-7} 
\multicolumn{1}{l|}{}                                                                         & $T_{real}$ & \multicolumn{1}{c|}{10.25s} & \multicolumn{1}{c|}{10.05s} & \multicolumn{1}{c|}{9.75s}  & \multicolumn{1}{c|}{9.95s} & \textbf{9.8s}  \\ \hline
\multicolumn{1}{l|}{\multirow{2}{*}{\begin{tabular}[c]{@{}l@{}}GP-Opt\\ +Track \cite{doublegp} \end{tabular}}} & $T_{plan}$ & \multicolumn{1}{c|}{8.96s}  & \multicolumn{1}{c|}{8.95s}  & \multicolumn{1}{c|}{8.94s}  & \multicolumn{1}{c|}{8.94s} & 8.95s          \\ \cline{2-7} 
\multicolumn{1}{l|}{}                                                                         & $T_{real}$ & \multicolumn{1}{c|}{9.8s}   & \multicolumn{1}{c|}{9.7s}   & \multicolumn{1}{c|}{9.55s}  & \multicolumn{1}{c|}{9.6s}  & \textbf{9.55s} \\ \hline
\multicolumn{1}{l|}{\multirow{2}{*}{Spline-based}}                                            & $T_{BO}$ & \multicolumn{1}{c|}{10.0s}  & \multicolumn{1}{c|}{9.95s}  & \multicolumn{1}{c|}{9.7s}   & \multicolumn{1}{c|}{9.8s}  & 9.55s          \\ \cline{2-7} 
\multicolumn{1}{l|}{}                                                                         & $T_{real}$ & \multicolumn{1}{c|}{10.6s}  & \multicolumn{1}{c|}{10.3s}  & \multicolumn{1}{c|}{10.15s} & \multicolumn{1}{c|}{10.0s} & \textbf{9.7s}  \\ \hline
\multicolumn{1}{l|}{\multirow{2}{*}{Non-iterative}}                                           & $T_{BO}$   & \multicolumn{5}{c}{9.25s}                                                                                                             \\ \cline{2-7} 
\multicolumn{1}{l|}{}                                                                         & $T_{real}$ & \multicolumn{5}{c}{\textbf{9.55s}}                                                                                                    \\ \hline
\multicolumn{1}{l|}{\multirow{2}{*}{Proposed}}                                                & $T_{BO}$   & \multicolumn{1}{c|}{8.75s}  & \multicolumn{1}{c|}{8.75s}  & \multicolumn{1}{c|}{9.0s}   & \multicolumn{1}{c|}{8.8s}  & 8.75s          \\ \cline{2-7} 
\multicolumn{1}{l|}{}                                                                         & $T_{real}$ & \multicolumn{1}{c|}{9.5s}   & \multicolumn{1}{c|}{9.1s}   & \multicolumn{1}{c|}{9.15s}  & \multicolumn{1}{c|}{8.95s} & \textbf{8.75s} \\ \hline
\end{tabular}%
}
\end{table}

\begin{figure}[!t]
    \centering
    \includegraphics[width=0.75\linewidth]{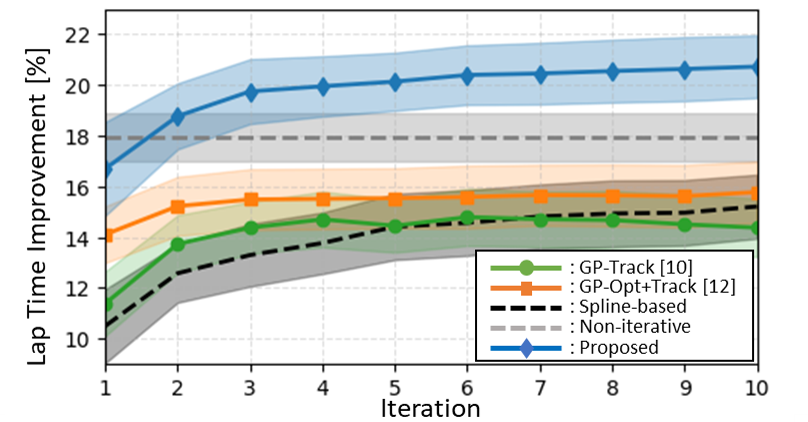}
    \caption{Average lap time improvement across randomly generated scenarios.}
    \label{fig:result_average}
\end{figure}

\subsection{Hardware Experiment}
\begin{figure}[!t]
    \centering
    \includegraphics[width=0.42\textwidth]{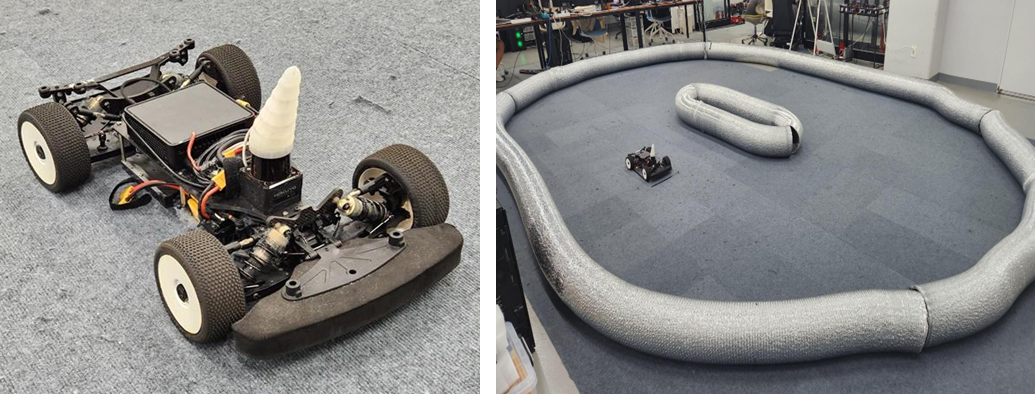}
    \caption{Experimental setup: (\emph{left}) the 1/10-scale vehicle platform and (\emph{right}) the indoor track.}
    \label{fig:experiment_setup}
\end{figure}

The proposed framework is further validated on a 1/10-scale autonomous racing platform, as shown in Fig.~\ref{fig:experiment_setup}.
The platform is equipped with an onboard Intel NUC computer running all algorithms within the ROS framework, including indoor localization and state estimation.
To evaluate the effectiveness of the simulation-enabled trajectory optimization, we consider two MPC tracking controllers: one with well-tuned cost weights and another with poorly-tuned weights that result in degraded tracking performance.
At each iteration, the optimized trajectory is deployed and evaluated through three repeated trials.
As shown in Fig.~\ref{fig:experiment_result}, the proposed framework achieves the lowest lap times across successive iterations under both controller settings.
Notably, similar lap times are observed across the two controller settings, demonstrating great compatibility of the proposed framework with different tracking controllers.

\begin{figure} [!t]
    \centering
    \includegraphics[width=0.47\textwidth]{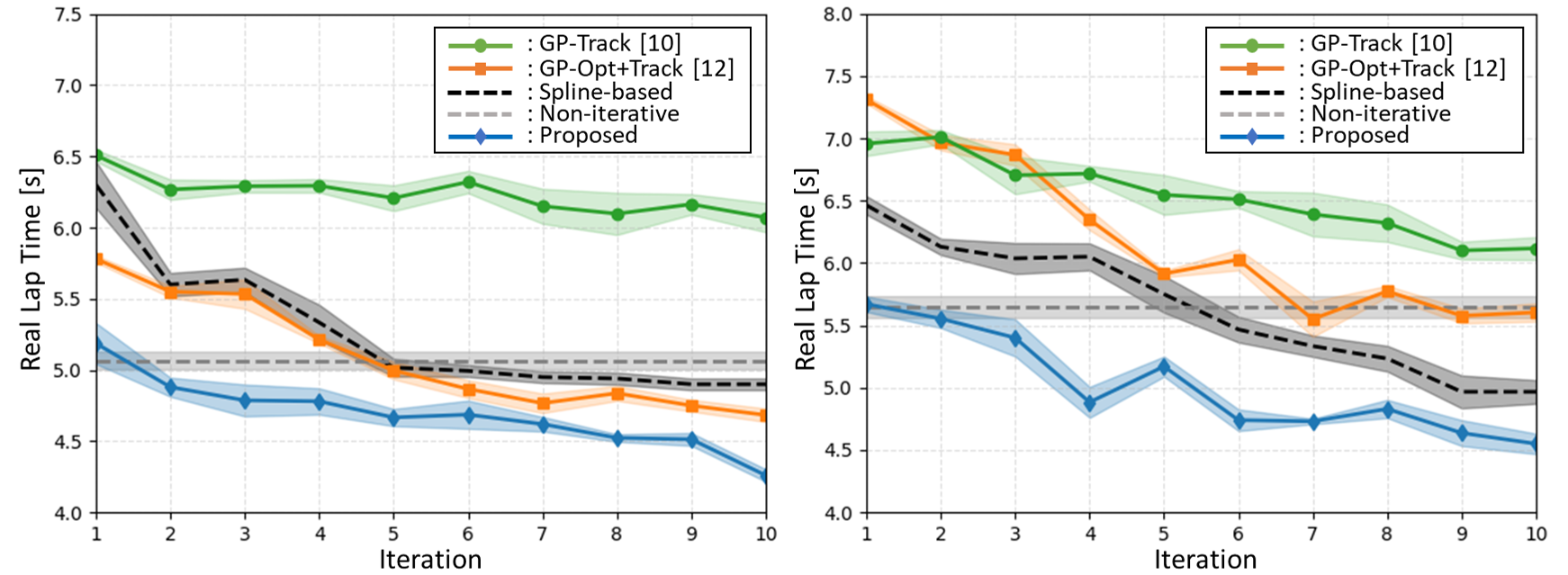}
    \caption{Experimental lap time results over iterations for different controller settings: (\emph{left}) the well-tuned and (\emph{right}) the poorly-tuned.}
\label{fig:experiment_result}
\end{figure}

\section{Conclusion}
This paper proposes a track-centric iterative learning framework for global trajectory optimization, particularly aiming at autonomous racing.
We first introduce a wavelet-based trajectory parameterization to represent full-horizon racing trajectories in a track-agnostic and low-dimensional space.
This space is then explored using BO, where candidate trajectories are evaluated by running simulations with learned vehicle dynamics.
By embedding the simulation-enabled BO into an iterative learning loop, the optimized trajectory is deployed on the real track to collect data, which is used to learn the vehicle dynamics. 
Then, the dynamics is updated over the iterations, progressively improving the simulation fidelity, thereby refining the trajectory.
The effectiveness of the proposed framework is demonstrated through extensive simulations and real-world experiments on a 1/10-scale autonomous racing platform, where it consistently achieves faster lap times than existing methods.
Future work will further investigate spatially varying characteristics of vehicle dynamics along the track to further enhance the performance of track-centric trajectory refinement.

\appendix
\subsection{Proof of Proposition \ref{prop:surrogate_fidelity}}
\label{appendix:proof_prop1}
The goal of this proof is to bound the discrepancy between the true lap time $J_{g^*}(\theta)$ and the BO-evaluated lap time $J_{g^j}(\theta)$ in terms of the maximum error $\epsilon^j$ of the model mismatch between true dynamics and learned dynamics.
To this end, we analyze how the state deviation between the true and simulated trajectories propagates over the discretized arc-length indices $k = 0, \dots, N_s-1$.
Let $\{\mathbf{x}_k, \mathbf{u}_k\}_{k=0}^{N_s-1}$ denote the true trajectory generated by $g^{*}$, and $\{\hat{\mathbf{x}}^j_k, \hat{\mathbf{u}}^j_k\}_{k=0}^{N_s-1}$ denote the simulated trajectory generated by $g^{j}$.
For notational convenience and without loss of generality, we redefine the residual dynamics term \eqref{eq.gp_model} as 
$g(\mathbf{x}, \mathbf{u}) := B_g g(\mathbf{z})$.
The corresponding discrete-space dynamics are then given by:
\begin{subequations}
\begin{align}
    \mathbf{x}_{k+1} &= \mathbf{x}_{k} + [f(\mathbf{x}_k, \mathbf{u}_k) + g^*(\mathbf{x}_k, \mathbf{u}_k)]\Delta s \label{eq.discrete-space dynamics1}\\
    \hat{\mathbf{x}}^j_{k+1} &= \hat{\mathbf{x}}^j_{k} + [f(\hat{\mathbf{x}}^j_k, \hat{\mathbf{u}}^j_k) + g^{j}(\hat{\mathbf{x}}^j_k, \hat{\mathbf{u}}^j_k)] \Delta s. \label{eq.discrete-space dynamics2}
\end{align}
\end{subequations}
 The state deviation is defined as $\delta \mathbf{x}^j_k := \mathbf{x}_k - \hat{\mathbf{x}}^j_k$. 
 Subtracting \eqref{eq.discrete-space dynamics1} from \eqref{eq.discrete-space dynamics2} gives:
\begin{equation} \label{eq.deviation_dynamics}
\begin{aligned}
    \delta \mathbf{x}^j_{k+1} = \delta \mathbf{x}^j_{k} + & \ [f(\mathbf{x}_k, \mathbf{u}_k) - f(\hat{\mathbf{x}}^j_k, \hat{\mathbf{u}}^j_k)]\Delta s \\
                     & + [g^*(\mathbf{x}_k, \mathbf{u}_k) - g^j(\hat{\mathbf{x}}^j_k, \hat{\mathbf{u}}^j_k)]\Delta s.
\end{aligned}
\end{equation}
With the dynamics error defined as $e_k^{j} := g^*(\hat{\mathbf{x}}^j_k, \hat{\mathbf{u}}^j_k) - g^{j}(\hat{\mathbf{x}}^j_k, \hat{\mathbf{u}}^j_k)$, a first-order Taylor expansion of \eqref{eq.deviation_dynamics} around $(\hat{\mathbf{x}}^j_k, \hat{\mathbf{u}}^j_k)$ yields:
\begin{equation}
\label{eq.linearized_dev}
\delta \mathbf{x}^j_{k+1}
\approx
\bigl( I + \Delta s\, M^j_k \bigr)\, \delta \mathbf{x}^j_k
+ \Delta s\, e_k^j
\end{equation}
where $M^j_k$ denotes the closed-loop Jacobian of the deviation dynamics.
Unrolling \eqref{eq.linearized_dev} from $k=0$, with $\delta \mathbf{x}^j_0 = 0$, gives:
\begin{equation}
    \delta \mathbf{x}^j_{k} \approx \Delta s \sum_{r=0}^{k-1} \Phi^j(k, r+1) e_r^j
\end{equation}
where $\Phi^j(k, r+1) = \prod_{m=r+1}^{k-1} (I + \Delta s\, M_m^j)$ is the state deviation transition matrix. 
Since all trajectories lie in the compact set $\mathcal{X}$, and true dynamics and the feedback policy are continuously differentiable,
the associated Jacobians are uniformly bounded.
Hence, under Assumption \ref{assum:model_convergence}, there exists a finite constant $G_k > 0$ such that:
\begin{equation}
\label{eq.state_bound}
    \|\delta \mathbf{x}^{j}_k\|
    \le 
    \Delta s \sum_{r=0}^{k-1} \|\Phi^j(k, r+1)\|\, \|e_r^{j}\|
    \le G_k\, \epsilon^{j}.
\end{equation}

Next, we relate the state deviation bound in \eqref{eq.state_bound} to the lap time prediction error.
Recalling \eqref{eq.continuous_dyn_a}, we write the longitudinal progress rate as
$\dot{s} = h(\mathbf{x})$.
According to the spatial formulation in Sec.~\ref{section:Mintime} and the positivity of the longitudinal speed, there exists a constant $v_{\min}>0$ such that:
\begin{align}
\label{eq.laptime_diff_bound}
    |J_{g^{j}}(\theta) - J_{g^*}(\theta)|
    &= \left| 
    \sum_{k=0}^{N_s-1} \frac{\Delta s}{\hat{\dot{s}}^j_k}
    - \sum_{k=0}^{N_s-1} \frac{\Delta s}{\dot{s}_k} 
    \right| \nonumber \\
    &\le \frac{\Delta s}{v_{\min}^2} 
    \sum_{k=0}^{N_s-1} 
    |\dot{s}_k - \hat{\dot{s}}^j_k|.
\end{align}
Since $h(\mathbf{x})$ is continuously differentiable, the rate deviation can be locally linearized and bounded based on \eqref{eq.state_bound} as:
\begin{equation}
\label{eq.lipschitz_rate}
\begin{aligned}
    |\dot{s}_k - \hat{\dot{s}}^j_k|
    &= |h(\mathbf{x}_k) - h(\hat{\mathbf{x}}^j_k)|
     \approx \Big\|\tfrac{\partial h}{\partial \mathbf{x}}\Big|_{\hat{\mathbf{x}}^j_k} \delta \mathbf{x}_k^{j}\Big\| \\
    &\le L_k \|\delta \mathbf{x}_k^{j}\|
    \le L_k G_k\, \epsilon^{j}.
\end{aligned}
\end{equation}
Finally, substituting \eqref{eq.lipschitz_rate} into \eqref{eq.laptime_diff_bound} yields:
\begin{equation*}
\label{eq.final_bound}
    |J_{g^{j}}(\theta) - J_{g^*}(\theta)|
    \le 
    \frac{\Delta s}{v_{\min}^2} 
    \sum_{k=0}^{N_s-1} L_k G_k\, \epsilon^{j} \nonumber \\
    = C^{j}\, \epsilon^{j}.
\end{equation*}
This completes the proof. 
\hfill $\square$

\bibliographystyle{IEEEtran}
\bibliography{reference}

\end{document}